\documentclass[conference]{IEEEtran}

\usepackage{soul}

\IEEEoverridecommandlockouts
\usepackage{cite}
\usepackage{amsmath,amssymb,amsfonts}
\usepackage{amsthm}
\usepackage{algorithmic}
\usepackage{graphicx}
\usepackage{textcomp}
\usepackage{xcolor}
\usepackage[]{footmisc}
\usepackage{bbm}
\usepackage{balance}

\usepackage{cuted}
\usepackage{stfloats}
\usepackage{mathtools,amssymb,lipsum, nccmath}
\newtheorem{theorem}{Theorem}
\newtheorem{lemma}[theorem]{Lemma}
\newtheorem{prop}{Proposition}
\usepackage{algorithm,algorithmic}
\usepackage{color,soul}
\usepackage{graphicx}
\usepackage{subfig}
\usepackage{hyperref}
\newtheorem{Ass}{Assumption}

\newtheorem{remark}{Remark}

\def\BibTeX{{\rm B\kern-.05em{\sc i\kern-.025em b}\kern-.08em
    T\kern-.1667em\lower.7ex\hbox{E}\kern-.125emX}}
\begin{document}

\title{Time-Varying Optimization for Streaming Data \\Via Temporal Weighting}

\author{%
  Muhammad Faraz Ul Abrar\textsuperscript{*}, 
  Nicol\`{o} Michelusi\textsuperscript{*}, 
  and Erik G. Larsson\textsuperscript{\dag} \\[1ex]
  \textsuperscript{*}School of Electrical, Computer and Energy Engineering, Arizona State University\\
  \textsuperscript{\dag}Department of Electrical Engineering (ISY), Link\"oping University\\[1ex]
  Emails: {mulabrar@asu.edu, nicolo.michelusi@asu.edu, 
                   erik.g.larsson@liu.se}
                   
\thanks{This research was funded in part by NSF under grant CNS-$2129015$. The work of E. G. Larsson was supported in part by ELLIIT, VR, and the KAW foundation. }}

\maketitle
\setulcolor{red}
\setul{red}{2pt}
\setstcolor{red}

\begin{abstract}
Classical optimization theory deals with fixed, time-invariant
objective functions. However, time-varying optimization has emerged as an important subject for decision-making in dynamic environments. In this work, we study the problem of learning from streaming data through a time-varying optimization lens. Unlike prior works that focus on generic formulations, we introduce a structured, \emph{weight-based} formulation that explicitly captures the streaming-data origin of the time-varying objective, where at each time step, an agent aims to minimize a weighted average loss over all the past data samples. We focus on two specific
weighting strategies: (1) uniform weights, which treat all samples
equally, and (2) discounted weights, which geometrically decay the influence of older data. For both schemes, we derive tight bounds on the ``tracking error'' (TE), defined as the deviation between the model parameter and the time-varying optimum at a given time step, under gradient descent (GD) updates. We show that under uniform weighting, the TE vanishes asymptotically with a $\mathcal{O}(1/t)$ decay rate, whereas discounted weighting incurs a nonzero error floor controlled by the discount factor and the number of gradient updates performed at each time step. Our theoretical findings are validated through numerical simulations.
\end{abstract}









\section{Introduction}

The deployment of machine learning (ML) solutions has surged across
diverse domains with applications in autonomous vehicles, robotics,
telecommunications, power grids, and cyber-physical
systems. Conventional ML optimizes a static objective, and therefore inherently assumes a static data distribution. Yet, real-world solutions operate under dynamically
evolving environments and must continuously adapt to the streaming
information
\cite{Time_Structured,Ali_Sayed2014Adaptation,CL_Survey}. Examples
include tracking a moving robot, localizing a mobile target, portfolio
optimization, risk management in fluctuating financial markets, and
adapting a controller with time-varying system dynamics. From a
learning perspective, this leads to a streaming data setting in which
the objective function evolves over time, resulting in a
non-stationary optimization problem. The goal then becomes to track
the optimum of a \emph{time-varying} objective function $F_t(\cdot)$:
\cite{Polyak_book,tracking_minimum_1998,Time_Structured}:
\begin{align}
\overline{\mathbf{w}}^*_t &= \arg \min_{\mathbf{w}\in \mathbb{R}^d} F_t(\mathbf{w}), \quad t\geq0.
\label{generic_TV}
\end{align}
This class of problems is typically approached using iterative
optimization techniques such as gradient descent (GD) or the Newton
method. 
However, due to computational constraints in real-time or resource-limited settings \cite{Time_Structured,DallAnese2019_stream}, only a limited number of updates can typically be performed at each time step, preventing exact tracking of $\overline{\mathbf{w}}^*_t$.
Consequently, a widely adopted performance metric in this
context is the ``tracking error" (TE) $\Vert
\mathbf{w}_t - \overline{\mathbf{w}}^*_t \Vert$, defined as the distance between the current iterate $\mathbf{w}_t$ and the time-varying optimum $\overline{\mathbf{w}}_t^*$.

Early work in \cite{tracking_minimum_1998}
proposed a Newton-type
algorithm leveraging second-order information to achieve exponential convergence in gradient norm for time-varying objectives. Subsequent studies have shown that for $\mu$-strongly convex and $L$-smooth objectives, GD can track the time-varying minimizer $\overline{\mathbf{w}}_t^*$ within an $\mathcal{O}(C)$ neighborhood, assuming a uniform bounded drift condition $\|\overline{\mathbf{w}}^*_{t+1} -
\overline{\mathbf{w}}^*_t\| \leq C$ \cite{Polyak_book,Time_Structured}. A different condition was later utilized in \cite{Popkov}, where convergence was established under the assumption $\|\nabla F_{t+1}(\mathbf{w})
- \nabla F_{t}(\mathbf{w})\| \leq C$ for all $\mathbf{w} \in \mathbb{R}^d$. Such “correction‐only” schemes rely solely on gradient- or Newton-type updates to track the drifting minimizer \cite{Time_Structured}.  By contrast,
prediction–correction schemes first forecast the next optimizer $\widehat{\mathbf{w}}_{t+1}$ using information up to time $t$ (for instance, via first-order optimality condition), and then apply gradient- or Newton-type correction steps once $F_{t+1}$ is revealed \cite{Class_Prediction_Correction}. Other approaches for prediction-based algorithms include the Kalman filter-based linear estimation and neural network-based non-linear estimation
\cite{Parameter_based_prediction}. More recently, time-varying
optimization has been studied in the distributed and decentralized
settings (e.g.,
\cite{Simonetto2016DecentralizedPM,Distributed_TV_Quadratic,D_cont_TV,hu2024energy}). However, most of these works focus on the generic formulation \eqref{generic_TV},
neglecting the inherent structure of streaming data, which may result in loose bounds on the TE.

A closely related paradigm is \emph{continual learning} (CL), which focuses on learning an ML model on a sequence of ``tasks'' \cite{Learn_without_forget,CL_Review,CL_Survey,CL_Survey_Defy_Forget}. Unlike conventional learning, CL assumes no access to future task data and only limited access to past task data samples. This constraint gives rise to the challenge of ``catastrophic forgetting", where learning new tasks degrades performance on earlier ones \cite{Catastrophic_McCloskey}. 
Although CL also addresses the challenge of adapting models to evolving data, existing approaches remain largely empirical.

In this work, we bridge time‐varying optimization and CL from a
theoretical standpoint. Specifically, we propose a structured
formulation in which the objective at time $t$ is defined as a
weighted average of the losses on all past samples. By explicitly
encoding temporal relevance in the weights, our framework captures the
streaming data structure and permits tight, weight-specific bounds on
TE performance. We discuss two natural weighting strategies: first, uniform weights, which assign equal weight to all past samples and model a stationary environment; second, \emph{discounted
weights}, which geometrically decay past sample contributions, prioritizing recent observations. For both schemes, we characterize the TE under GD updates. Exploiting the
streaming‐data structure yields sharper TE bounds than
the existing generic time‐varying analyses. Specifically, with uniform
weights, we prove that the TE decays as $\mathcal{O}(1/t)$. Under
discounted weights, we derive an explicit, nonzero asymptotic TE bound
that quantifies the impact of the discount factor and the number of GD
iterations per time step.

The rest of the paper is organized as follows. Section \ref{sec:system_model} presents the system model. Section \ref{sec:TE_analysis} develops the TE analysis for the uniform and discounted weighting schemes. Section \ref{sec:results} reports numerical results, and Section \ref{sec:conclusion} concludes the work.

\section{System Model}
\label{sec:system_model}
We consider a learning setup where a single agent (e.g., an edge or cloud server) aims to track a time-varying ML model parameter. In particular, we assume that data arrives sequentially in an online streaming fashion, where at every iteration $t{\geq}1$, the agent receives a new data sample 
$(\mathbf{x}_t, y_t)$. Here, $ \mathbf{x}_t$  denotes the feature vector and $ y_t$ represents the corresponding label. Let $f_t(\cdot)$ denote the loss associated with the data sample arrived in iteration $t$, such as the cross-entropy loss evaluated on the $t$-th sample. The goal is to obtain an ML model that minimizes the \emph{weighted} average loss over the data accumulated thus far, formulated as:
\begin{align}
   \overline{\mathbf{w}}^*_t = \arg \min_{\mathbf{w} \in \mathbb{R}^d} F_t(\mathbf{w}), \text{where } F_t(\mathbf{w})\triangleq \sum_{i=1}^t a_i(t) f_i(\mathbf{w}).
    \label{average_time_vary_prob}
\end{align}
Here, $a_i(t)$ are the weights associated with $f_i(\cdot)$ assigned at iteration $t$. We assume that $ 0 \leq a_i(t)\leq 1 $ for all $i$ and $t$ and $\sum_{i=1}^t a_i(t) =1$ for all $t$, so that $F_t(\cdot)$ represents a convex combination of the individual losses associated to each data point accumulated up to time $t$. We note that, unlike the generic time‐varying formulation \eqref{generic_TV}, \eqref{average_time_vary_prob} explicitly captures the fact that the objective arises from streaming data.

We employ the gradient descent (GD) method with fixed step size to solve \eqref{average_time_vary_prob}. For each time step $t$, the agent performs $E$ gradient updates, which are initialized with $\mathbf{w}_{t,0} = \mathbf{w}_{t}$. These updates are of the form:
\begin{align}
\label{local_GD}
\mathbf{w}_{t,k+1} &= 
\mathbf{w}_{t,k} - \eta \nabla F_{t+1}(\mathbf{w}_{t,k}) \\&
=\mathbf{w}_{t,k} - \eta \sum_{i=1}^{t+1} a_i(t+1) \nabla f_i(\mathbf{w}_{t,k})\,, \label{local_GD_1}
\end{align}
for $k = 0,1,\cdots,E-1$, where $\eta$ represents the learning step size. Finally, the updated model is obtained as $\mathbf{w}_{t+1} = \mathbf{w}_{t,E}$.
Note that the model update in \eqref{local_GD_1} requires computing the gradients for all historical samples $\{\nabla f_i(\cdot)\}_{i \leq t+1}$. The focus of this work is to understand the fundamental limits of structured time-varying learning, and therefore, we do not assume memory constraints. Under this assumption, $\nabla F_{t+1}(\cdot)$ can be exactly computed at each time step. Analysis of memory-constrained time-varying learning is reserved for future investigation.
Since the global objective $F_t(\cdot)$ keeps evolving over time, we characterize the learning performance using the ``tracking error" (TE), defined as:
\begin{align}
    \text{TE}(t) =  \Vert \mathbf{w}_t - \overline{\mathbf{w}}^*_t  \Vert, 
\end{align}
for all $t \geq 1$. Here we are interested in analyzing how the TE evolves over time, and its dependence on $E$. 
We are also interested in the asymptotic tracking error (ATE) $\triangleq \limsup_{t \to \infty} \Vert \mathbf{w}_t - \overline{\mathbf{w}}^*_t  \Vert$\footnote{Note that we used $\limsup$, since the limit may not exist.} to understand the convergence behavior of model updates in \eqref{local_GD} and the corresponding rate of convergence. These questions are addressed next.

\section{Tracking Error Analysis}
\label{sec:TE_analysis}
In this section, we characterize the tracking error associated with the GD updates in \eqref{local_GD} to solve the time-varying learning problem in \eqref{average_time_vary_prob} for a given budget on the number of GD iterations $E$. In particular, we investigate the TE for two special choices of weights $\{a_i(t)\}$: (a) uniform weights, where each sample is assigned equal weights, and (b) discounted weights, in which past samples are assigned geometrically discounted weights. The first case models a stationary data environment in which all past data samples are equally important. In contrast, the latter captures a scenario in which recent observations are deemed more relevant than older ones, as is common in evolving data dynamics.
The TE for both choices is analyzed under the following assumptions:
\begin{Ass}
\label{ass:smooth_sc}
The loss function associated with each data sample is $L$-smooth (has Lipschitz-continuous gradients) and $\mu$-strongly convex, that is, for all $t = 1,2,\cdots,$ $f_t(\cdot)$ satisfies
\begin{align}\| \nabla f_t(\mathbf{x}) - \nabla f_t(\mathbf{y})\| \leq L\| \mathbf{x} - \mathbf{y}\|,\end{align}
\begin{align}
f_t(\mathbf{y}) \geq f_t(\mathbf{x}) + \nabla f_t(\mathbf{x})^\top(\mathbf{y} - \mathbf{x}) + \frac{\mu}{2}\| \mathbf{y} - \mathbf{x} \|^2,  
\end{align}  
for all $\mathbf{x},\mathbf{y} \in \mathbb{R}^d$. Since $F_t(\cdot)$ is a convex combination of the losses up to time $t$, it is also $L$-smooth and $\mu$-strongly convex for all $t \geq 1$.
\end{Ass}
\begin{Ass}
\label{ass:bounded_minimizers}
The minimizers of the data sample losses, $\mathbf{w}^*_t = \arg \min_{\mathbf{w} \in \mathbb{R}^d} f_t(\mathbf{w)}$, are uniformly bounded, that is, $\exists \;C >0$ such that $\Vert\mathbf{w}^*_t \Vert \leq C $, for all $t$. 
\end{Ass}
Assumption~\ref{ass:smooth_sc} is standard in analyzing gradient-based methods (see, e.g., \cite{Nesterov,boyd2004convex}). 
In contrast, Assumption~\ref{ass:bounded_minimizers} (bounded minimizers) is more specific. This assumption 
was used before in the analysis of time-varying optimization for the federated learning setting (see Assumption 5 in \cite{hu2024energy}) and is required to make the global optimizer
$\overline{\mathbf{w}}^*_t$ drift slowly enough with $t$. In fact, as we will show in \eqref{minimizer_drift_uniform} and \eqref{minimizer_drift_discount}, it implies the often‐invoked “bounded drift” assumption  $\|\overline{\mathbf{w}}^*_{t+1}-\overline{\mathbf{w}}^*_t\|\le \bar{C}$
commonly adopted in prior time‐varying optimization work (e.g. \cite{Time_Structured}),
but is a more primitive and transparent structural assumption on the problem. 

Under Assumption ~\ref{ass:smooth_sc}, the TE at iteration $t+1$ can be analyzed as follows. From \eqref{local_GD}, we have
\begin{align}
&\text{TE}(t+1)
=\Vert \mathbf{w}_{t+1} - \overline{\mathbf{w}}^*_{t+1}  \Vert = \Vert \mathbf{w}_{t,E} - \overline{\mathbf{w}}^*_{t+1}  \Vert\\
&=\Vert \mathbf{w}_{t,E-1} - \overline{\mathbf{w}}^*_{t+1}
- \eta[\nabla F_{t+1}(\mathbf{w}_{t,E-1}) -\nabla F_{t+1}(\overline{\mathbf{w}}^*_{t+1})
]
\Vert  \nonumber\\
&\overset{(a)}{\leq} (1 - \eta \mu) \Vert \mathbf{w}_{t,E-1} - \overline{\mathbf{w}}^*_{t+1}\Vert \nonumber\\
&= (1 - \eta \mu) \Big\Vert \mathbf{w}_{t,E-2}-  \overline{\mathbf{w}}^*_{t+1}\\&
-\eta[\nabla F_{t+1}(\mathbf{w}_{t,E-2})
-\nabla F_{t+1}(\overline{\mathbf{w}}^*_{t+1})]
\Big\Vert \nonumber\\
&\overset{(b)}{\leq} (1 - \eta \mu)^2 \Vert \mathbf{w}_{t,E-2} - \overline{\mathbf{w}}^*_{t+1}\Vert  \nonumber\\
    &\hspace{2mm}\vdots
    \nonumber\\
    &\leq(1 - \eta \mu)^E \Vert \mathbf{w}_{t} - \overline{\mathbf{w}}^*_{t+1}\Vert,
    \label{TE_analysis}
\end{align}
where $\mathbf{w}_{t,0} = \mathbf{w}_t$. Steps $(a)$ and $(b)$ leverage the fact that $\nabla F_{t+1}(\overline{\mathbf{w}}^*_{t+1})=\mathbf 0$ along with Lemma~\ref{lemma:MHT} (provided in the Appendix), which exploits the $\mu$-strong convexity and $L$-smoothness of $F_{t+1}(\cdot)$ (Assumption~\ref{ass:smooth_sc}) for a learning step size choice $\eta \in (0, \frac{2}{\mu + L}]$. The final inequality in \eqref{TE_analysis} is obtained using induction over $E$ gradient updates. Using the triangle inequality, we can further bound the error term in \eqref{TE_analysis} as
\begin{align*}    
\Vert \mathbf{w}_{t} - \overline{\mathbf{w}}^*_{t+1}\Vert&\leq\Vert \mathbf{w}_{t} - \overline{\mathbf{w}}^*_{t}\Vert
+
\Vert\overline{\mathbf{w}}^*_{t+1}-\overline{\mathbf{w}}^*_{t}\Vert\\
&=\text{TE}(t)+
\Vert\overline{\mathbf{w}}^*_{t+1}-\overline{\mathbf{w}}^*_{t}\Vert.
\end{align*}
Hence, continuing from 
\eqref{TE_analysis} and letting $\alpha \triangleq (1-\eta \mu)^E$, we obtain the following recursive expression on the TE:
\begin{align}
\text{TE}(t+1)& \leq \alpha \left( \text{TE}(t)+\left\Vert \overline{\mathbf{w}}^*_{t+1} - \overline{\mathbf{w}}^*_{t} \right\Vert \right). \label{TE_analysis_1}
\end{align}
\begin{remark}
    The expression in \eqref{TE_analysis_1} reveals that the TE follows a contracting sequence (since $\alpha <1$) except for the presence of an extra term $\alpha \Vert \overline{\mathbf{w}}^*_{t+1} - \overline{\mathbf{w}}^*_{t} \Vert$ capturing the scaled drift of the minimizers of the time-varying objectives.
\end{remark}
Using recursion, the TE in \eqref{TE_analysis_1} can be further expressed as 
\begin{align}
     \text{TE}(t)&\leq \alpha^{t-1}\text{TE}(1) + \sum_{i=1}^{t-1}\alpha^{t-i} \left\Vert \overline{\mathbf{w}}^*_{i+1} - \overline{\mathbf{w}}^*_{i} \right\Vert\\
     &\leq
\alpha^{t}\Vert \mathbf{w}_{0} - \overline{\mathbf{w}}^*_{1}\Vert+ \sum_{i=1}^{t-1}\alpha^{t-i} \left\Vert \overline{\mathbf{w}}^*_{i+1} - \overline{\mathbf{w}}^*_{i} \right\Vert,\label{TE_recursive}
\end{align}
where we bounded TE$(1)$ via \eqref{TE_analysis} to obtain the last inequality. The above bound on the TE contains two components: 1) the impact of the initialization, which diminishes geometrically with $t$, and 2) the error accumulated due to the drift of the minimizers of the time-varying objective functions.
Next, we proceed to bound the drift of the minimizers $\left\Vert \overline{\mathbf{w}}^*_{i+1} - \overline{\mathbf{w}}^*_{i} \right\Vert$. Utilizing the $\mu$-strong convexity of $F_{i+1}(\cdot)$, we obtain:
\begin{align}
    \Vert  \overline{\mathbf{w}}^*_{i+1} - \overline{\mathbf{w}}^*_{i} \Vert &\leq \frac{1}{\mu} \Vert \nabla F_{i+1}( \overline{\mathbf{w}}^*_{i+1}) -\nabla F_{i+1}( \overline{\mathbf{w}}^*_{i})\Vert \nonumber\\ 
    &= \frac{1}{\mu} \Vert \nabla F_{i+1}( \overline{\mathbf{w}}^*_{i})\Vert, \label{minimizer_drift}
\end{align}
where the equality follows due to the optimality condition $ \nabla F_{i+1}( \overline{\mathbf{w}}^*_{i+1}) = \mathbf{0}$. Since $F_t(\cdot)$ and the corresponding minimizer $\overline{\mathbf{w}}^*_{t}$  depends on the choice of weights $\{a_i(t)\}$, further bounding the minimizer drift in \eqref{minimizer_drift} necessitates specializing the analysis to specific weighting schemes, which is done in the next two subsections for uniform and discounted weights, respectively.

\subsection{Uniform Weights}
A natural strategy is to assign uniform weights to each data sample observed thus far, and hence set $a_i(t) = 1/t$ for all $i = 1,\cdots,t$, and for all $t$. In this case, we can express $F_{t+1}(\cdot)$ for any $t \geq 1$ as
\begin{align}
    F_{t+1}(\mathbf{w}) &= \sum_{i=1}^{t+1} \frac{f_i(\mathbf{w})}{t+1}     \nonumber
    =  \frac{1}{t+1} \sum_{i=1}^t  f_i(\mathbf{w}) + \frac{1}{t+1} f_{t+1}(\mathbf{w})\nonumber\\
    &= \frac{t}{t+1} F_t(\mathbf{w}) + \frac{1}{t+1}f_{t+1}(\mathbf{w}). \label{Recursive_uniform}
\end{align}
Continuing from \eqref{minimizer_drift} and using \eqref{Recursive_uniform}, we can bound the minimizer drift 
as
\begin{align}
    \Vert  \overline{\mathbf{w}}^*_{t+1} - \overline{\mathbf{w}}^*_{t} \Vert  &\leq \frac{1}{\mu} \left\Vert \frac{t}{t+1} \nabla F_t(\overline{\mathbf{w}}^*_{t}) + \frac{1}{t+1}\nabla f_{t+1}(\overline{\mathbf{w}}^*_{t})\right\Vert \nonumber \\
    &= \frac{1}{\mu(t+1)} \left\Vert \nabla f_{t+1}(\overline{\mathbf{w}}^*_{t})\right\Vert \nonumber  \\
     &\overset{(a)}{\leq} \frac{L}{\mu(t+1)} \left\Vert \overline{\mathbf{w}}^*_{t} -\mathbf{w}_{t+1}^*\right\Vert \nonumber\\
     &\overset{(b)}{\leq} \frac{L}{\mu(t+1)} \left\Vert \overline{\mathbf{w}}^*_{t}\right\Vert + \frac{L}{\mu(t+1)} \left\Vert\mathbf{w}_{t+1}^*\right\Vert \nonumber\\
     &\overset{(c)}{\leq}  \frac{1}{t+1} \bigg(1 + \sqrt{\frac{L}{\mu}}\bigg) \frac{L C}{\mu}=\frac{C'}{t+1},\label{minimizer_drift_uniform}
\end{align}
where the equality follows due to the optimality condition $ \nabla F_{t}( \overline{\mathbf{w}}^*_{t}) = \mathbf{0}$,
and in the last step we defined $C' \triangleq \Big(1{+} \sqrt{\frac{L}{\mu}}\Big) \frac{L C}{\mu}$.
 Step $(a)$ follows from $L$-smoothness of $f_{t+1}(\cdot)$ (Assumption~\ref{ass:smooth_sc}), $(b)$ uses the triangle inequality, and $(c)$ utilizes Assumption~\ref{ass:bounded_minimizers} and Lemma ~\ref{lemma:glob_minimizer_bnd} provided in the Appendix. Using the bound on the minimizer drift in \eqref{minimizer_drift_uniform} along with \eqref{TE_recursive}, the TE for uniform weights can be expressed as 
\begin{align}
    \text{TE}(t)&\leq \alpha^{t}\Vert \mathbf{w}_{0} - \overline{\mathbf{w}}^*_{1}\Vert+ C'\sum_{i=1}^{t-1}  \frac{\alpha^{t-i}}{i+1
    }. \label{uniform_TE}
\end{align}
Next, we will show that for large $t$, the sum in \eqref{uniform_TE} admits an $\mathcal{O}(1/t)$ upper bound.
\begin{prop}
\label{prop:uniform_sum}
Define $S(t)\;\triangleq\;\sum_{i=1}^{t-1}\frac{\alpha^{t-i}}{\,i+1\,}\,,$ and let $A=\max\{t_0 S(t_0), \frac{2\alpha}{1-\alpha}\}$, with  $t_0 = \lceil \frac{2\alpha}{1-\alpha}\rceil$. Then for all $t\geq t_0 $, 
$$S(t)\;\leq\;\frac{A}{t}.$$ 
\end{prop}
\begin{proof}
We can express $S(t+1)$ as
    \begin{align}
        S(t+1) &=\sum_{i=1}^{t}  \frac{\alpha^{t+1-i}}{i+1}= \frac{\alpha}{t+1} + \sum_{i=1}^{t-1}  \frac{\alpha^{t+1-i}}{i+1}\nonumber\\
        &= \frac{\alpha}{t+1} + \alpha \sum_{i=1}^{t-1}  \frac{\alpha^{t-i}}{i+1} = \frac{\alpha}{t+1} + \alpha S(t).\label{Uniform_Sum_recursive}
    \end{align}
By the definition of $A$, we have $A \geq t_0S(t_0)$; therefore it directly holds that $S(t_0) \leq \frac{A}{t_0}$. Next, as induction hypothesis we assume that $S(t) \leq \frac{A}{t}$ for some $t\geq t_0 = \lceil \frac{2\alpha}{1-\alpha}\rceil$. 
    Then, $S(t+1)$ can be upper bounded from \eqref{Uniform_Sum_recursive} as $S(t+1) \leq \frac{\alpha}{t+1} + \alpha \frac{A}{t}$. To prove this, it is sufficient to show that the right-hand side above is bounded by $\frac{A}{t+1}$. Reorganizing the terms, this is equivalent to showing that 
    \begin{align*}
        A\big(1 - \alpha -\frac{\alpha}{t}\big)\geq\alpha.
    \end{align*}
    Since $t\geq t_0\geq 2\alpha/(1-\alpha)$, it follows that 
    $1-\alpha-\alpha/t\geq(1-\alpha)/2$, hence a sufficient condition to satisfy the previous condition is
    $
    A\geq \frac{2\alpha}{1-\alpha},
    $
which holds by definition of $A$.
    We have thus proved that $S(t+1)\leq A/(t+1)$. By induction, the bound holds for all $t \geq t_0$, completing the proof.
\end{proof}

Using Proposition~\ref{prop:uniform_sum}, the TE for the uniform weights in \eqref{uniform_TE} can be bounded as
\begin{align}
    \text{TE}(t)&\leq \alpha^{t}\Vert \mathbf{w}_{0} - \overline{\mathbf{w}}^*_{1}\Vert + C'\frac{A}{t}, \ \forall t\geq t_0,\label{uniform_TE_final}
\end{align}
where $A$ and $t_0$ are given in the statement of Proposition~\ref{prop:uniform_sum}.
\begin{remark} Since the first term in \eqref{uniform_TE_final} decays geometrically, it can be concluded using Proposition ~\ref{prop:uniform_sum} that the TE for the uniform weights decays as $\mathcal{O}\left(\frac{1}{t}\right)$ for sufficiently large $t$, and therefore a vanishing asymptotic TE, i.e., $\lim_{t\to\infty} \text{TE}(t) = 0$ is achieved. Moreover, we can lower bound $S(t)$ as $S(t) \geq \frac{1}{t}\sum_{i=1}^{t-1}\alpha^{t-i} = \frac{\alpha}{t}(\frac{1-\alpha^{t-1}}{1-\alpha})$, which also decays as $\mathcal{O}(\frac{1}{t})$ for $t$ large enough. The matching upper and lower bounds confirm that 
$\mathcal{O}(1/t)$ cannot be improved. Remarkably, the 
$\mathcal{O}(1/t)$ convergence to the time-varying minimizer $\overline{\mathbf{w}}^*_{t}$ holds even when the sequence $\{\overline{\mathbf{w}}^*_{t}\}$ itself is non-convergent. 
\end{remark}
\subsection{Discounted Weights}
Another strategy is to geometrically discount the samples observed in the past, i.e., $a_i(t) \propto \gamma^{t-i}$ for all $i\leq t$, where $0<\gamma<1$ is the discount factor. To ensure that $\sum_{i=1}^t a_i(t) =1$, we normalize the weights yielding:
\begin{align}
    a_i(t) = \frac{1-\gamma}{1-\gamma^t} \gamma^{t-i},\, \forall\; i =1,\cdots,t.
\end{align}
Accordingly, for the discounted weights, we can express the global objective at iteration $t+1$ as 
\begin{align}
    &F_{t+1}(\mathbf{w})= \sum_{i=1}^{t+1} \frac{1-\gamma}{1-\gamma^{t+1}} \gamma^{t+1-i} f_i(\mathbf{w})  \nonumber\\
    &= \sum_{i=1}^t \frac{1-\gamma}{1-\gamma^{t+1}} \gamma^{t+1-i} f_i(\mathbf{w})  + \frac{1-\gamma}{1-\gamma^{t+1}} f_{t+1}(\mathbf{w})  \nonumber\\
    &= \gamma \cdot \frac{1-\gamma^t}{1-\gamma^{t+1}} \left( \sum_{i=1}^{t} \frac{1-\gamma}{1-\gamma^t} \gamma^{t-i} f_i(\mathbf{w}) \right) + \frac{1-\gamma}{1-\gamma^{t+1}} f_{t+1}(\mathbf{w}) \nonumber\\
    &= \gamma \cdot \frac{1-\gamma^t}{1-\gamma^{t+1}} F_t(\mathbf{w}) + \frac{1-\gamma}{1-\gamma^{t+1}} f_{t+1}(\mathbf{w}).
    \label{Recursive_discount}
\end{align}
Using \eqref{Recursive_discount}, the minimizer drift in \eqref{minimizer_drift} specializes as
\begin{align}
     &\Vert  \overline{\mathbf{w}}^*_{t+1}{-} \overline{\mathbf{w}}^*_{t} \Vert{\leq} \frac{1}{\mu} \Big\Vert \frac{\gamma(1{-}\gamma^t)}{1{-}\gamma^{t+1}}  \nabla F_t(\overline{\mathbf{w}}^*_{t}) {+}  \frac{1{-}\gamma}{1{-}\gamma^{t+1}} \nabla f_{t+1}(\overline{\mathbf{w}}^*_{t})\Big\Vert \nonumber \\
    &= \frac{1-\gamma}{\mu(1-\gamma^{t+1})} \left\Vert \nabla f_{t+1}(\overline{\mathbf{w}}^*_{t})\right\Vert \leq  \frac{1-\gamma}{1-\gamma^{t+1}} 
     C'
     ,\label{minimizer_drift_discount}
\end{align}
with $C'$ defined as in \eqref{minimizer_drift_uniform}, where the equality uses the optimality condition $\nabla F_t(\overline{\mathbf{w}}^*_t)=0$, and the final
inequality follows from Assumptions \ref{ass:smooth_sc}–\ref{ass:bounded_minimizers}, Lemma \ref{lemma:glob_minimizer_bnd}, and the triangle inequality, in direct analogy to the steps used to obtain \eqref{minimizer_drift_uniform}. Using the minimizer drift bound in \eqref{minimizer_drift_discount}, the TE in \eqref{TE_recursive} specializes to
\begin{align}
    \text{TE}(t)\leq \alpha^{t}\Vert \mathbf{w}_{0} - \overline{\mathbf{w}}^*_{1}\Vert + C' (1-\gamma)\sum_{i=1}^{t-1} \frac{\alpha^{t-i}}{1-\gamma^{i+1}}.
    \label{discount_TE}
\end{align}
\begin{prop}
\label{prop:discount_sum}
Define $S(t) \triangleq \sum_{i=1}^{t-1}  \frac{(1-\gamma)\alpha^{t-i}}{1 - \gamma^{i+1}
    }$ and let $A_\gamma=\max\{\frac{(1-\gamma^{t_0})S(t_0)}{1-\gamma}, \frac{2\alpha}{1-\alpha}\}$, and  $ t_0= \lceil\ln(\frac{1-\alpha}{1+\alpha-2\gamma\alpha})/\ln(\gamma)\rceil$. Then for all $t\geq t_0 $ we have     
$$ S(t) \leq \frac{A_\gamma(1-\gamma)}{1-\gamma^t}. $$ 
Furthermore, $$\lim_{t \to \infty} S(t) = \frac{(1-\gamma)\alpha}{1-\alpha}.$$
\end{prop}
\begin{proof}
 We begin by computing 
 $S(t+1) = \sum_{i=1}^{t}  \frac{(1-\gamma)\alpha^{t+1-i}}{1 - \gamma^{i+1}} $, which can also be recursively expressed using $S(t)$ as
    \begin{align}
        &S(t+1) =  \frac{(1-\gamma)\alpha}{1-\gamma^{t+1}} + \sum_{i=1}^{t-1}  \frac{(1-\gamma)\alpha^{t+1-i}}{1 - \gamma^{i+1}}\nonumber\nonumber\\
        &\!\!{=}\frac{(1-\gamma)\alpha}{1-\gamma^{t+1}} + \alpha \sum_{i=1}^{t-1}  \frac{(1-\gamma)\alpha^{t-i}}{1 - \gamma^{i+1}} = \frac{(1-\gamma)\alpha}{1-\gamma^{t+1}} + \alpha S(t).\label{Discount_Sum_recursive}
    \end{align}
Note that by definition of $A_\gamma$, we have $A_\gamma \geq \frac{(1-\gamma^{t_0})S(t_0)}{1-\gamma}$, therefore it directly holds that $S(t_0) \leq \frac{A_\gamma(1-\gamma)}{1-\gamma^{t_0}}$. Next, we use the induction hypothesis and assume that $ S(t) \leq \frac{A_\gamma(1-\gamma)}{1-\gamma^t}$ for some arbitrary $t\geq t_0 =\lceil\ln(\frac{1-\alpha}{1+\alpha-2\gamma\alpha})/\ln(\gamma)\rceil$. Then, $S(t+1)$ can be upper bounded as $S(t+1) \leq  \frac{(1-\gamma)\alpha}{1-\gamma^{t+1}} + \alpha  \frac{A_\gamma(1-\gamma)}{1-\gamma^t}$. To prove the induction, it suffices to show that 
    the right-hand side above is bounded by $\frac{A_\gamma(1-\gamma)}{1-\gamma^{t+1}}$. Equivalently, after reorganizing the terms and simplifying:
\begin{align}
    A_\gamma\Big[
1-\alpha-\alpha\gamma^t\frac{
1-\gamma}{1-\gamma^t}
\Big]
\geq \alpha.
\label{acond2}
\end{align}
    To show that this condition holds true, note that for $t\geq t_0$ and since $\gamma<1$,
    $$
    \gamma^t\leq \gamma^{t_0}
    \leq
    \gamma^{\ln(\frac{1-\alpha}{1+\alpha-2\gamma\alpha})/\ln(\gamma)}
    =
\frac{1-\alpha}{1+\alpha-2\gamma\alpha}.
    $$
    Therefore, we can lower bound the left-hand side of \eqref{acond2} as
    $$
    A_\gamma\Big[
1-\alpha-\alpha\gamma^t\frac{
1-\gamma}{1-\gamma^t}
\Big]
$$$$
\geq
A_\gamma\Big[
1-\alpha-\alpha
\frac{1-\alpha}{1+\alpha-2\gamma\alpha}
\frac{
1-\gamma}{1-\frac{1-\alpha}{1+\alpha-2\gamma\alpha}}
\Big]
=
A_\gamma\frac{1-\alpha}{2},
    $$
    where the inequality is due to the bound on $\gamma^t$.
    Finally, by the definition of 
    $A_\gamma$, it holds that 
    $A_\gamma(\frac{1-\alpha}{2})\geq \alpha$, 
    so that \eqref{acond2} is satisfied.
We have thus proved that $S(t+1)\leq \frac{A_\gamma(1-\gamma)}{1-\gamma^{t+1}}$. By induction, the bound holds for all $t \geq t_0$.

Next, we apply Lemma \ref{lemma:convergent_seq} from the Appendix to the sequence $S(t)$ governed by \eqref{Discount_Sum_recursive} with  $b_t \triangleq \frac{(1-\gamma)\alpha}{1-\gamma^{t+1}}$. Since $b_t{\to}(1{-}\gamma)\alpha$ as $t\to \infty$, it immediately follows from this lemma that $\lim_{t \to \infty} S(t) = \frac{(1-\gamma)\alpha}{1-\alpha}$, which completes the proof.
\end{proof}
    
Using Proposition~\ref{prop:discount_sum}, the TE for the discounted weights in \eqref{discount_TE} can be bounded as
\begin{align}
    \text{TE}(t)&\leq \alpha^{t}\Vert \mathbf{w}_{0} - \overline{\mathbf{w}}^*_{1}\Vert + C' \frac{A_\gamma(1-\gamma)}{1-\gamma^t}, \ \forall t\geq t_0 \label{discount_TE_final}
\end{align}
with $A_\gamma$ and $t_0$ defined in Proposition~\ref{prop:discount_sum}.
\begin{remark}
\label{remark:ATE_discount}Since the first term in \eqref{discount_TE} decays geometrically, whereas the second term is a non-vanishing term, it can be concluded using Proposition~\ref{prop:discount_sum} that, with discounted weights, a non-vanishing ATE is achieved,\,i.e., 
\begin{align}
    \text{ATE}_\gamma \triangleq \lim\sup_{t \to \infty} \Vert \mathbf{w}_t - \overline{\mathbf{w}}^*_t  \Vert \leq   \left(1{+} \sqrt{\frac{L}{\mu}}\right) \frac{L C}{\mu}\cdot \frac{(1-\gamma)\alpha}{1-\alpha }.
\end{align}
    \end{remark}
\begin{remark}
\label{Remark:gamma_to_one}
Uniform weights treat every past sample equally, so as $t$ grows, the influence of any one new sample on the overall objective decays like $\mathcal{O}(1/t)$.  Equivalently, the minimizer drifts by $\mathcal{O}(1/t)$ each step (see \eqref{minimizer_drift_uniform}), and this drift, and hence the TE, vanishes asymptotically.  By contrast, with $\gamma$-discounting, old samples are exponentially forgotten.  The minimizer, therefore, continues to drift by an amount proportional to $1-\gamma=\mathcal{O}(1)$ whenever a new sample arrives (see \eqref{minimizer_drift_discount}), so the algorithm never fully catches up, causing a non-vanishing ATE. Moreover, it is interesting to note that as $\gamma \to 1$ the discounted weights converge to uniform weights, and correspondingly 
$\text{ATE}_\gamma \to 0$.
\end{remark}

\begin{prop}
    Let $0 < \gamma <1$, $0<\eta \leq \frac{2}{\mu + L }$, and $C' >0$. For a given choice of $\epsilon >0$, one can ensure $\text{ATE}_\gamma\leq \epsilon$ by performing 
    \begin{align}
         E \geq \frac{\ln\left(\frac{\epsilon}{C'(1-\gamma) + \epsilon}\right)}{\ln(1- \eta\mu)}
    \label{iter_complexity}
    \end{align}
    gradient updates in each time step. This leads to a gradient iteration complexity of $\mathcal{O}(\ln(\frac{1}{\epsilon}))$ (when $\epsilon\ll C'(1-\gamma)$) to achieve $\epsilon-$ATE.
\end{prop}

\section{Numerical Results}
\label{sec:results}
In this section, we perform experimentation to numerically verify the efficacy of our TE analysis
for the considered time-varying learning problem. We investigate the TE for both the uniform and discounted weights using scalar quadratic loss functions. 
In particular, the loss function associated with the data arrived in iteration $t$ is given by 
\begin{align*}
     f_t(w)= \frac{\mu}{2} \left(w - c_t\right)^2.
\end{align*}
It is straightforward to see that $f_t(\cdot)$ is minimized at $w_t^* = c_t$, and the global objective $F_t(\cdot)$ is minimized at $\overline{w}^*_{t} = \sum_{i=1}^t a_i(t) c_i$, where $\{a_i(t)\}$ denote the weighting coefficients. We generate $\{c_t\}$ according to a bounded random walk process, that has the form:
\begin{align}
    c_{t+1} = \max(-C_\text{max},\min(c_t + z_{t+1},C_\text{max})), 
\end{align}
for all $t = 0,1,\cdots,$ where we let $z_t \sim N(0,\sigma^2)$ for all $t\geq1$, and we set $c_0 = 0$. We set $C_\text{max} = 100$, $\sigma^2 = 100$, and $\mu = 0.1$. Note that the above choice of loss functions $\{f_t\}$ ensures that the global objective $F_t(\cdot)$ satisfies both Assumptions~\ref{ass:smooth_sc} and ~\ref{ass:bounded_minimizers} with $C=C_\text{max}$.

\begin{figure}[t]
\centering
\includegraphics[width=0.45\textwidth]{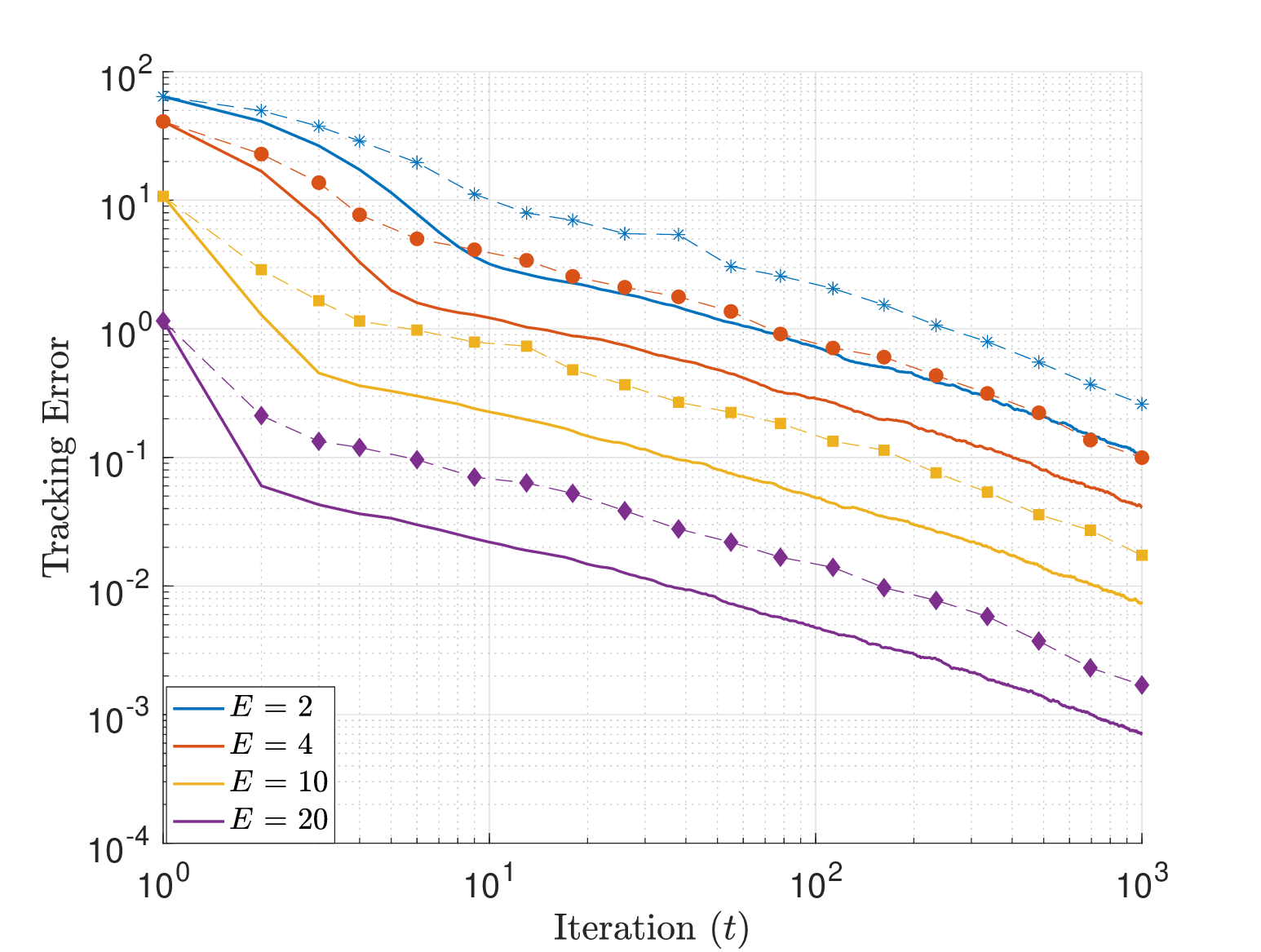}
\caption{Root mean squared (\textbf{solid lines}) and maximum tracking error (\textbf{dashed lines with markers}) vs. iterations for uniform weights, $\mu=0.1, \eta=2$. 
}
\label{Fig:TE_uniform}
\end{figure}

Fig.~\ref{Fig:TE_uniform} plots the root mean squared (RMS) and the maximum (worst-case) TE (over 1000 independent and identical runs) vs. iterations for the case of uniform weights, considering different choices of the number of gradient updates $E$. As the iteration index $t$ increases, both the RMS and the worst‐case TE curves clearly exhibit a monotonic decay consistent with the $\mathcal{O}(1/t)$ rate established in equation \eqref{uniform_TE_final}. As expected, increasing the number of gradient updates $E$ accelerates the TE decay rate. This is because a larger $E$ reduces the contraction factor $\alpha = (1 - \eta\mu)^E$, thereby diminishing both the initialization error and the accumulated drift-induced error in \eqref{uniform_TE_final} more rapidly. We also note that for large $t$, the TE becomes dominated by the residual error from minimizer drift, which scales as $\mathcal{O}(1/t)$. Notably, doubling $E$ from 10 to 20 reduces the empirical TE by an improvement factor of $(1 - \eta\mu)^{20-10} \big[\frac{(1 - \eta\mu)^{10}}{(1 - \eta\mu)^{20}}\big]\approx 0.09$ (with $\eta = 2, \mu = 0.1$), quantitatively matching the theoretical prediction according to \eqref{uniform_TE_final}.
\begin{figure}[t]
\centering
\includegraphics[width=0.45\textwidth]{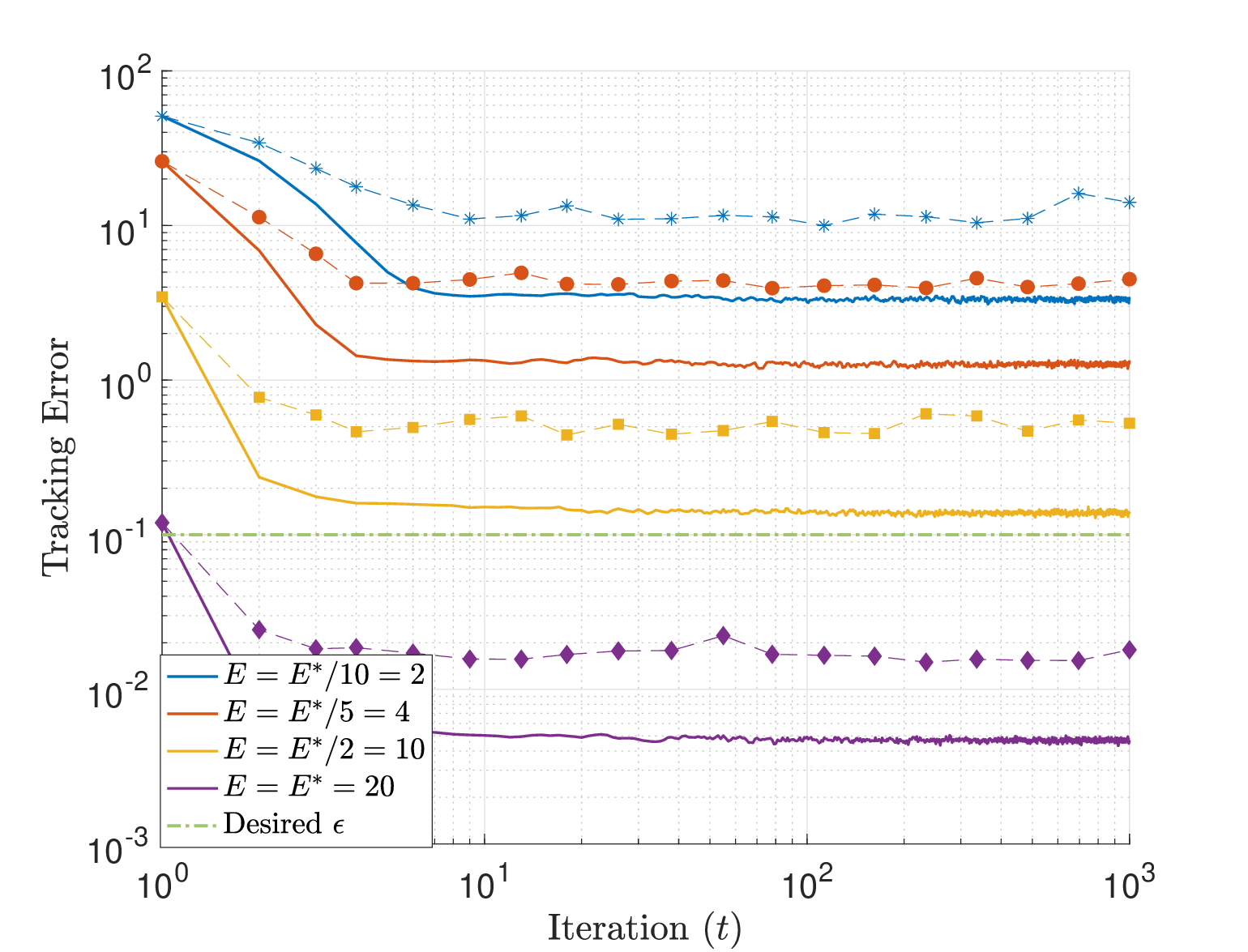}
\caption{Root mean squared (\textbf{solid lines}) and maximum tracking error (\textbf{dashed lines with markers}) vs. iterations for discounted weights, $\gamma = 0.7, \mu=0.1, \eta=2.85$.}
\label{Fig:TE_discount}
\end{figure}

RMS and the worst-case TE for the discounted case are plotted in Fig.~\ref{Fig:TE_discount} with different choices of the number of gradient updates $E$. The minimum $E^*$ derived from~\eqref{iter_complexity} to guarantee $\epsilon = 0.1$ ATE ensures the error floor remains below $\epsilon$, whereas the considered choices with $E < E^*$ violate the desired condition. Consistent with Remark~\ref{remark:ATE_discount}, both error metrics converge to a non-vanishing asymptotic error floor. Finally, we observe that doubling $E$ from 10 to 20 causes a drop of $(1 - \eta\mu)^{20-10} \big[\frac{(1 - \eta\mu)^{10}}{(1 - \eta\mu)^{20}}\big]\approx 0.03$ (with $\eta = 2.85, \mu = 0.1$) in the ATE, which matches with our theoretical bound in \eqref{discount_TE_final}. 
\begin{figure}[t]
\centering
\includegraphics[width=0.45\textwidth]{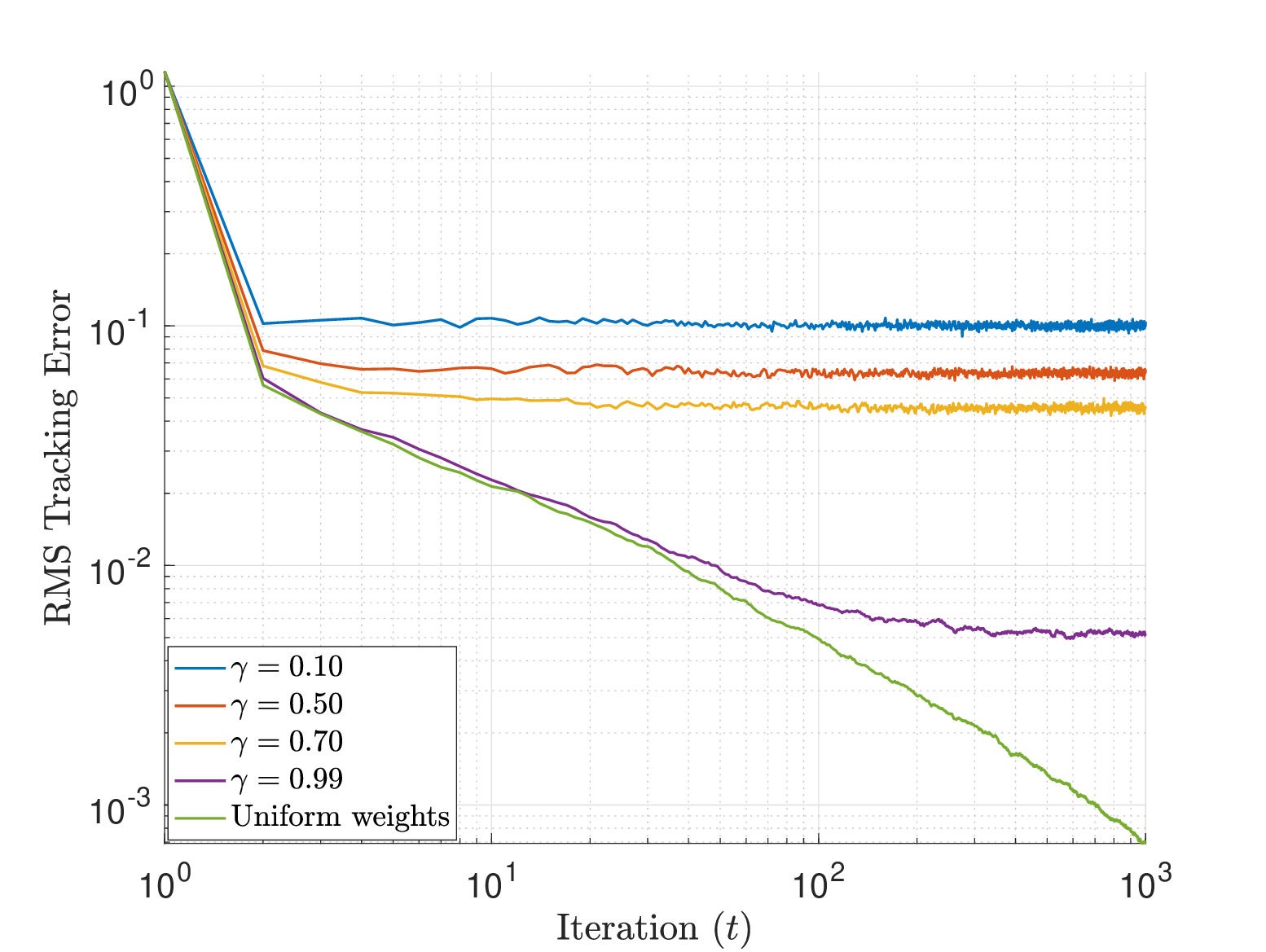}
\caption{Root mean squared TE vs. iterations for varying discount factor $\gamma$.}
\label{Fig:TE_discount_diff_gamma}
\end{figure}

Fig.~\ref{Fig:TE_discount_diff_gamma} plots the TE against iterations for the discounted weights for various choices of the discount factor $\gamma$. Although all curves approach a nonzero asymptotic error, larger $\gamma$ yields a lower ATE. This behavior matches the intuition that smaller $\gamma$ discounts the past samples more heavily, causing the minimizer to drift more rapidly and consequently a higher TE. Fig.~\ref{Fig:TE_discount_diff_gamma} also verifies Remark \ref{Remark:gamma_to_one} empirically, where for $\gamma=0.99$, the TE closely matches the uniform weight case.

\section{Conclusion}
\label{sec:conclusion}
In this paper, we established theoretical guarantees for time-varying optimization specialized to the ML setting where data arrives in an online streaming fashion. We captured this by formulating an objective that evolves as a weighted average of past losses. We analyzed gradient-descent updates under two canonical weighting schemes, uniform weights, which yield a vanishing TE that decays as $\mathcal{O}(1/t)$, and geometrically discounted weights, leading to a non-vanishing asymptotic TE floor explicitly characterized by the discount factor and gradient iterations. Our weight-specific analysis yields tighter, interpretable expressions that clarify how the weighting scheme and the per-time-step gradient update budget shape the asymptotic error, and we confirm these predictions through numerical simulations.


\appendix
\begin{lemma}
\label{lemma:MHT}
Let $g:\mathbb{R}^d\to\mathbb{R}$ be a differentiable function that satisfies Assumption \ref{ass:smooth_sc}. Then,  for all $\mathbf{x},\mathbf{y} \in \mathbb{R}^d$, there exists a positive definite matrix $\mathbf{A}$ such that 
\begin{align*}
\nabla g(\mathbf{y})-\nabla g(\mathbf{x})=\mathbf{A}(\mathbf{y}-\mathbf{x}),
\end{align*} 
where the eigenvalues of $\mathbf{A}$ lie in the interval $ [\mu, L] $. Moreover, for $ 0 < \eta \leq \frac{2}{\mu + L}$ it holds that 
\begin{align*}
\Vert \mathbf{y} - \mathbf{x} - \eta \left( \nabla g(\mathbf{y}) - \nabla g(\mathbf{x})\right)\Vert \leq (1-\eta\mu)\Vert  \mathbf{y} - \mathbf{x}\Vert.
\end{align*}
\end{lemma}
\begin{proof}
This is the mean Hessian theorem, stated and proved in Section II.E of \cite{MHT}.
\end{proof}
\begin{lemma}
\label{lemma:glob_minimizer_bnd}
Under Assumptions \ref{ass:smooth_sc} and \ref{ass:bounded_minimizers}, the set of global minimizers $\{\overline{\mathbf{w}}^*_{t}\}_{t\geq1}$ is uniformly bounded and satisfies:
\begin{align}
    \Vert \overline{\mathbf{w}}^*_{t} \Vert^2 
    &\leq \frac{L}{\mu} C^2,
\end{align}
\end{lemma}
\begin{proof}
Using the $\mu$-strong convexity of $F_t(\cdot)$, we have
\begin{align}
    \Vert \overline{\mathbf{w}}^*_{t} \Vert^2 \leq \frac{2}{\mu}\left(F_t(\mathbf{0}) - F_t(\overline{\mathbf{w}}^*_{t})\, \right)
    \label{global_minimizer_bnd}
\end{align}
Furthermore, by using $L$-smoothness of $f_i(\cdot)$, we have
\begin{align*}
    f_i(\mathbf{0}) \leq f_i(\mathbf{w}_i^*) + \frac{L}{2} \Vert \mathbf{w}_i^*\Vert^2 \leq f_i(\mathbf{w}_i^*) + \frac{L}{2} C^2,
\end{align*}
where the second inequality invokes Assumption~\ref{ass:bounded_minimizers}. Averaging over $i=1,\cdots,t$ with weights $a_i(t)\geq0$ and $\sum_{i=1}^ta_i(t) =1$, we obtain:
\begin{align}
    F_t(\mathbf{0})&\leq\sum_{i=1}^t a_i(t)f_i(\mathbf{w}_i^*) + \frac{L}{2} C^2 \nonumber\leq \sum_{i=1}^t a_i(t)f_i(\overline{\mathbf{w}}^*_{t}) + \frac{L}{2} C^2, \nonumber
\end{align}
where the second inequality follows since $\mathbf{w}_i^*$ minimizes $f_i(\cdot)$, which combined with \eqref{global_minimizer_bnd} completes the proof. 
\end{proof}
\begin{lemma}
\label{lemma:convergent_seq}
Let $0<\alpha<1$ and let $\{b_t\}_{t\ge0}$ be a sequence of real numbers that satisfies $b_t\to b^*$ as $t \to \infty$.
Define the sequence $\{x_t\}_{t\ge0}$ by
$$
x_{t+1}=\alpha x_t+ b_t,\qquad x_0\in\mathbb{R}.
$$
Then, $\{x_t\}$ is also a convergent sequence which satisfies 
\[
\lim_{t\to\infty}x_t =\frac{b^*}{1-\alpha}\,.
\]
\end{lemma}

\begin{proof}
By induction on $t$, we can express $x_t$ as 
\begin{equation}\label{eq:unroll}
x_t=\alpha^t x_0+\sum_{k=0}^{t-1}\alpha^{\,t-1-k}b_k
\quad ,\;\;t\geq 1.
\end{equation}
Since $b_k \to b^*$, we express $b_k$ as $b_k=b^*+e_k$ with $e_k \triangleq b_k-b^*\to0$. Subtracting the candidate limit $\frac{b^*}{1-\alpha}$ from both sides of \eqref{eq:unroll}, we have
\[
\begin{aligned}
x_t-\frac{b^*}{1-\alpha}
&=\alpha^t x_0
+\sum_{k=0}^{t-1}\alpha^{t-1-k}e_k
+ b^* \sum_{k=0}^{t-1}\alpha^{t-1-k} -\frac{b^*}{1-\alpha}\\
 &= \alpha^t(x_0 - \frac{b^*}{1-\alpha}) +\sum_{k=0}^{t-1}\alpha^{t-1-k}e_k  \\
&\triangleq A_t+B_t.
\end{aligned}
\]
Next, we will show that for every $\varepsilon>0$ there exists a $\tau$ such that
$|A_t|+|B_t|<\varepsilon$ for all $t \geq \tau$, that is $x_t \to \frac{b^*}{1-\alpha}$.
Let  $\varepsilon >0$ be given. Since $\alpha^t\to0$, there exists some $\tau_1 > 0$ such that
$|A_t|=\alpha^t|x_0 - \frac{b^*}{1-\alpha}|<\varepsilon/2$ for all $t\geq \tau_1$. Next, to bound $|B_t|$, we split the term $B_t$ as $B_t=H_t+T_t$ where
$$
H_t\triangleq \sum_{k=0}^{N-1}\alpha^{t-1-k}e_k,
\qquad
T_t \triangleq\sum_{k=N}^{t-1}\alpha^{t-1-k}e_k, \quad t \geq 1.
$$
Define $\delta \triangleq \frac{(1-\alpha)\varepsilon}{4} >0$. Since $e_k\to0$, there
exists  $N$ such that
$|e_k|\leq\delta\,, \forall k\geq N.$ With this, the tail term $T_t$ can be bounded as 
$$
|T_t|
\leq \delta\sum_{k=N}^{t-1}\alpha^{t-1-k} 
\leq \delta\sum_{j=0}^{\infty}\alpha^j
=\frac{\delta}{1-\alpha}
=\frac{\varepsilon}{4}.
$$
To bound the head term $H_t$, we use the fact that, since $e_t\to 0$, it is bounded by some $\mathcal{E} >0$ ($|e_t|\leq \mathcal{E},\forall t$), so that 
$$
|H_t|
\leq \sum_{k=0}^{N-1}\alpha^{t-1-k} |e_k| \leq  \mathcal{E} \sum_{k=0}^{N-1}\alpha^{t-1-k}
\leq \frac{\ \mathcal{E}}{1-\alpha}\,\alpha^{t-N}.
$$
Next, since $ \alpha^{t-N} \to 0, \text{ as } t \to \infty$, there exists a $\tau_2\geq N$ such that $|H_t|\leq \varepsilon/4$ for all $t\geq \tau_2$. Finally, choosing $\tau = \max\{\tau_1,\tau_2\}$, the following holds for all $t\geq \tau$,
$$
|x_t-\tfrac{b^*}{1-\alpha}|  
\leq |A_t|+|H_t|+|T_t|
< \frac{\varepsilon}{2}+\frac{\varepsilon}{4}+\frac{\varepsilon}{4}
=\varepsilon.
$$
Overall, we have proved that, for any $\varepsilon>0$, there exists a $\tau$ such that $\forall t\geq \tau$, $|x_t-\tfrac{b^*}{1-\alpha}|<\varepsilon$, thus proving the lemma.
\end{proof}

\bibliographystyle{IEEEtran}
\bibliography{sample}
\balance
\end{document}